\documentclass{article}
\usepackage[utf8]{inputenc}
\usepackage{hyperref} 
\usepackage{times}  
\usepackage{helvet}  
\usepackage{courier}  
\usepackage{url}  
\usepackage{graphicx}  
\makeatletter

\usepackage{amssymb}
\usepackage{amsmath,amsthm}
\usepackage{float}
\usepackage{url}
\usepackage{setspace}
\usepackage{hyperref}
\usepackage{todonotes}
\usepackage{amsthm} 
\usepackage[ruled,vlined,linesnumbered]{algorithm2e} 
\usepackage{makecell}
\usepackage{upgreek}
\usepackage{pdfpages}
\usepackage{times}
\usepackage{multirow}
\usepackage[toc,page]{appendix}
\usepackage{listings}
\newtheorem{thm}{Theorem}
\newtheorem{lemma}[thm]{Lemma}
\newtheorem{definition}[thm]{Definition}

\newtheorem{theorem}[thm]{Theorem}

\newtheorem{example}[thm]{Example}
\newtheorem{proposition}[thm]{Proposition}

\usepackage{color,soul}
\usepackage{tikz}
\usetikzlibrary{arrows,shapes.geometric,shapes.arrows}
\usepackage{pgfplots}
\usepackage{multirow}

\DeclareMathOperator{\supp}{support}
\DeclareMathOperator{\support}{support}

\SetKwFunction{mEqOne}{mEqOne}

\pgfplotsset{compat=1.5}
\begin{document}

\title{Efficient One Sided Kolmogorov Approximation}
\author{Liat Cohen$^1$  Tal Grinshpoun$^2$ Gera Weiss$^3$\\
        $^1$University of Basel
        $^2$Ariel University
		$^3$Ben-Gurion University of the Negev}
\maketitle

\begin{abstract}
We present an efficient algorithm that, given a discrete random variable $X$ and a number $m$, computes a random variable whose support is of size at most $m$ and whose Kolmogorov distance from $X$ is minimal, also for the one-sided Kolmogorov approximation. We present some variants of the algorithm, analyse their correctness and computational complexity, and present a detailed empirical evaluation that shows how they performs in practice. The main application that we examine, which is our motivation for this work, is estimation of the probability missing deadlines in series-parallel schedules. Since exact computation of these probabilities is NP-hard, we propose to use the algorithms described in this paper to obtain an approximation. 
\end{abstract}

\section{Introduction}

Various approaches for approximation of probability distributions are studied in the literature~\cite{PS77,AMCR83,vidyasagar2012metric,cohen2015estimating,pavlikov2016cvar,CohenGW18}. 
These approaches vary in the types random variables considered, how they are represented, and in the criteria used for evaluation of the quality of the approximations. In this paper we propose an approach for compressing the probability mass function of a random variable $X$ such that the errors added to queries such as $Pr(X\leq t)$, for  any $t>0$, is minimal. In other words, we minimise the Kolmogorov distance between the approximation and the original variable, see alternative definition in Equation~\eqref{eq:linear}. 

Our main motivation for this work is estimation of the probability for missing deadlines, as described, e.g., in Cohen et al.~\cite{cohen2015estimating,CohenGW18} and in~\cite{Kashef18}. Specifically, when $X$ represents the probability distribution of the time to complete some complex schedule and we cannot afford to maintain the full table of its probability mass function, we propose an algorithm for producing a smaller table, whose size can be specified, such that probabilities for missing deadlines are preserved as much as possible. 

The main contribution of this paper is an efficient algorithm for computing the best possible approximation of a given random variable with a random variable whose size is not above a prescribed threshold, where the measures of the quality of the approximation and of its size are as specified in the following two paragraphs.

We measure the quality of an approximation scheme by the distance between random variables and their approximations. Specifically, we use the Kolmogorov distance which is  commonly used for comparing random variables in statistical practice and literature. Given two random variables $X$ and $X'$ whose cumulative distribution functions (cdf) are $F_X$ and $F_{X'}$, respectively, the Kolmogorov distance between $X$ and $X'$ is $d_K(X,X')= \sup_t |F_X(t) - F_{X'}(t)|$ (see, e.g.,~\cite{gibbons2011nonparametric}). We say that $X'$ is a good approximation of $X$ if $d_K(X,X')$ is small. This distance is the basis for the often used Kolmogorov-Smirnoff test for comparing a sample to a distribution or two samples to each other. 

The size of a random variable is measured by the size of its support, the set of possible outcomes, $|X|{=}|\{x\colon Pr(X{=}x) \neq 0\}|$. When probability mass functions are maintained as tables, as done in many implementations of statistical software, the support size is proportional to the memory needed to store the variable and to the complexity of the computations that manipulate it. The exact notion of optimality of the approximation targeted in this paper is:

\begin{definition}
	A random variable $X'$ is an optimal $m$-approximation of a random variable $X$ if $|X'| \leq m$ and there is no random variable $X''$ such that $|X''| \leq m$ and $d_K(X,X'') < d_K(X,X')$.
\end{definition}

In these terms, the main contribution of the paper is an efficient (linear time and constant memory) algorithm that takes $X$ and $m$ as parameters and constructs an optimal $m$-approximation of $X$.

The rest of the paper is organised as follows. In Section we describe how our work relates to other algorithms and problems studied in the literature. In Section we detail the proposed algorithm, analyse its properties, and prove the main theorems. In Section we demonstrate how the proposed approach performs on the problem of estimating the probability of missing deadlines in series-parallel schedules on randomly generated random variables and compare it to alternative approximation approaches from the literature. The paper is concluded with a discussion and with ideas for future work in Section.

\section{Related work}\label{sec:relwork}
The most relevant work related to this paper is the papers on approximations of random variables in the context of estimating deadlines~\cite{cohen2015estimating,CohenGW18}. In these papers, $X'$ is defined to be a good approximation of $X$ if $F_{X'}(t) > F_{X}(t)$ for any $t$ and $\sup_t F_{X'}(t) - F_{X}(t)$ is small. Note that this measure is not a proper distance measure because it is not symmetric. The motivation given in these papers for using this type of approximation is for cases where overestimation of the probability of missing a deadline is acceptable but underestimation is not. We consider in this paper the same case-studies examined by Cohen et al. and show how the algorithm proposed in this paper performs relative to the algorithms proposed there when both over- and under- estimations are allowed. As expected, the Kolmogorov distance between the approximated and the original random variable is considerably smaller when using the algorithm proposed in this paper. 

In the technical level, the problem we study in this paper is similar to the problem of approximating a set of 2-D points by a step function. The study of this problem was motivated by query optimisation and histogram constructions in database management systems~\cite{applf12,applf13,applf14,applf17,applf18,Fournier2011} and computational geometry~\cite{diaz2001fitting,fournier2008fitting}. There are, however, two technically significant differences between the problem studied in the context of databases and the problem we analyse in this paper. The first difference is that in the context of approximation of random variables, the step function (which is the cumulative distribution function in our context)  must end with a value one, since we are dealing with random variables which sums to one. The second difference is that the first step is not counted because there is no need to put a value in the support of the approximated random variable to generate this first step. These cannot be addressed by adding a constant (two) to $m$ because the first step is always present and because the requirement to end with the value one, restricts the set of eligible step functions. 

Another relevant prior work is the theory of Sparse Approximation (aka Sparse Representation) that deals with sparse solutions for systems of linear equations, as follows. 
Given a matrix $D \in \mathbb{R}^{n \times p}$ and a vector $x \in \mathbb{R}^n$, the most studied sparse representation problem is finding 
$$
\min_{\alpha \in \mathbb{R}^p} \|\alpha\|_0 \text{ subject to } x = D\alpha
$$
where $\|\alpha\|_0 = |\{ i \in [p]: \alpha_i \neq 0 \}|$ is the $\ell_0$ pseudo-norm, counting the number of non-zero coordinates of $\alpha$. This problem is known to be NP-hard with a reduction to NP-complete subset selection problems.
In these terms, using also the $\ell_\infty$ norm that represents the maximal coordinate and the $\ell_1$ norm that represents the sum of the coordinates, our problem can be phrased as:
\begin{equation}\label{eq:linear}
\min_{\alpha \in [0,\infty)^p}\|x - D\alpha\|_{\infty} \text{ subject to }  \|\alpha\|_0 = m \text{ and } \|\alpha\|_1=1
\end{equation}
where $D$ is the lower unitriangular matrix, $x$ is related to $X$ such that the $i$th coordinate of $x$ is $F_X(x_i)$ where $\support(X)=\{x_1 < \cdots < x_n\}$ and $\alpha$ is related to $X'$ such that the $i$th coordinate of $\alpha$ is $f_{X'}(x_i)$. The functions $F_X$ and $f_{X'}$ represent, respectively, the cumulative distribution function of $X$ and the mass distribution function of $X'$, i.e.,  the coordinates of $x$ are positive and monotonically increasing and its last coordinate is one. 

The presented work is also related to the research on binning in statistical inference. Consider, for example, the problem of credit scoring~\cite{zeng2017comparison} that deals with separating good applicants from bad applicants where the Kolmogorov–Smirnov statistic KS is a standard measure. The KS comparison is often preceded by a procedure called binning where small values in the probability mass function are moved to nearby values. There are many methods for binning~\cite{mays2001handbook,refaat2011credit,bolton2010logistic,siddiqi2012credit}.
In this context, our algorithm can be considered as a binning strategy that provides optimality guarantees with respect to the Kolmogorov distance.

Our study is also related to the work of Pavlikov and Uryasev~\cite{pavlikov2016cvar}, where a procedure for producing a random variable $X'$ that optimally approximates a random variable $X$ is presented. Their approximation scheme, achieved using linear programming, is designed for a different notion of distance called CVaR. The contribution of the present work in this context is that our method is direct, not using linear programming, thus allowing tighter analysis of time and memory complexities. Also, our method is designed for minimising the Kolmogorov distance that is more prevalent in applications. For comparison, in Section we briefly discuss the performance of linear programming approach similar to the one proposed in~\cite{pavlikov2016cvar} for the Kolmogorov distance and compare it our algorithm. 

A problem very similar to ours is termed ``order reduction'' by Vidyasagar in~\cite{vidyasagar2012metric}. There, the author defines an information-theoretic based distance between discrete random variables and studies the problem of finding a variable whose support is of size $m$ and its distance from $X$ is as small as possible (where $X$ and $m$ are given). The main difference between this and the problem studied in this paper, is that Vidyasagar examines a different notion of distance. Vidyasagar proves that computing the distance (that he considers) between two probability distributions, and computing the optimal reduced order approximation, are both NP-hard problems, because they can both be reduced to nonstandard bin-packing problems. He then develops efficient greedy approximation algorithms. In contrast, our study shows that there are efficient solutions to these problems when the Kolmogorov distance is considered.

\begin{definition}
	A random variable $X'$ is an optimal one-sided $m$-approximation of a random variable $X$ if $|X'| \leq m$ and $F_X(t)<F_{X'}(t)$ for all $t$ and there is no random variable $X''$ such that $|X''| \leq m$ and $F_X(t)<F_{X''}(t)$ for all $t$ and $d_K(X,X'') < d_K(X,X')$.
\end{definition}

\begin{algorithm}
	\DontPrintSemicolon
    $f \gets 1$\; 
    \lWhile{$c_{f} \leq \varepsilon$}{$f \gets f + 1$}
	$S \gets \emptyset$, $s \gets 0$, $e \gets f$\;
	\While{$e<n+1$}{
	    \While{$c_{e+1}-c_{f} \leq \varepsilon \wedge e <n+1$}{$e \gets e+1$}
	    $S \gets S \cup \{(x_{f}, c_e) \}$\;
        $s \gets  c_e$, $f \gets e\gets e+1$\;
	}
	\If {$s<1$}{
	$S \gets S \cup \{(x_{f}, 1 - s) \}$\;}
    \Return{A r.v. $X'$ as CDF such that $Pr(X'{\leq}x)=c$ if there is $c$ such that $(x, c) \in S$.}

	\caption{$dual(\{(x_i, c_i)\}_{i=1}^n,\varepsilon)$  }   
	\label{alg:dual}
\end{algorithm}

\SetKwRepeat{Do}{do}{while}

\begin{example}
\label{exmpl:dual}
When $dual$ is invoked with the parameters  $X{=}\{(1, 0.3), (2, 0.7), (3, 0.9), (4,1)\}$ and $\varepsilon{=}0.1$: Line 1 $f=1$. After an iteration of the main loop (line 4), $S = \{(1,0.3)\}$, $s=0.3$, and $f=e=2$. After a second iteration, $S = \{(1,0.3), (2, 0.7)\}$, $s=0.7$, and $f=e=4$. At the end, $S = \{(1,0.3), (2, 0.7), (4, 1)\}$.
\end{example}

\begin{proposition}\label{the:correctnessDual}
	If  $p=\{(x_i, c_i)\}_{i=1}^n$ is such that $c_i=Pr(X \leq x_i)$ and $\supp(X)=\{x_1 < \cdots < x_n\}$ then  
    $$dual(\{(x_i, c_i)\}_{i=1}^n,\varepsilon) \in \underset{X'\in \bar{\mathcal{B}}(X,\varepsilon)}{\operatorname{arg\,min}}\, |X'|$$  where $\bar{\mathcal{B}}(X,\varepsilon)=\{X'\colon d_K(X,X')\leq \varepsilon \text{ and } F_X(x)<F_{X'}(x) \text{ for all } x\}$.
\end{proposition}
\begin{proof}
The number of level sets of the CDF of $X'$ is minimal: (1) The first set is of maximal length; (2) By construction, an extension of any of the other sets to the right will generate a random variable whose Kolmogorov distance from $X$ is bigger than $\varepsilon$. Thus, there is no random variable whose Kolmogorov distance from $X$ is smaller or equal to $\varepsilon$ and its support is smaller than $|X'|$. In addition, we keep the invariant $f \leq e$ and since $X$ is given as CDF, $c_f \leq c_e$ consistent throughout the algorithm.
\end{proof}

\begin{proposition}\label{the:complexityDual}
	$dual(\{(x_i, c_i)\}_{i=1}^n,\varepsilon)$ runs in time $O(n)$, using $O(n)$ memory.
\end{proposition}
\begin{proof}
The algorithm describes a single pass over $\{(x_i, c_i)\}_{i=1}^n$. Line 1 is easy to follow, and takes $O(n)$ in the worst case. Lines 3-7 also describe a single pass since the counter $e$ is updated to $e+1$ at most $n$ times. All together we get run-time complexity of $O(n)$. We are constructing the set $S$ which is of size $n$ in the worst case, therefore, memory complexity is $O(n)$. 
\end{proof}

\begin{algorithm}
	\DontPrintSemicolon
	Let $\{(x_i, c_i)\}_{i=1}^n$ be such that $c_i=Pr(X \leq x_i)$ and $\supp(X)=\{x_1 < \cdots < x_n\}$.\;
    $E \gets \emptyset$\; 
    \For{$i \gets 1$ \textbf{to} $n-1$}{
    	\For{$j \gets i+1$ \textbf{to} $n-1$}{
        		\If{$i=1$} {
                	$E \gets E \cup \{ c_j \}$
                }
        		$E \gets E \cup \{ c_j-c_i \}$
        }
    }
    Let $e_1<\cdots<e_{n'}$ be such that $E=\{e_1,\dots,e_{n'}\}$ \;
    
    \For{$i \gets 1 \textbf{ to } 2$}{
        $l\gets 1$, $r \gets n'$, $k \gets 1$, $k' \gets 0$ \;
        \While{$k \neq k'$}{
                $k' \gets k$,$k \gets \lceil (l + r)/2 \rceil$\;
        		$m' \gets |dual(X,e_k)|$\;
        		\lIf{$m' < m$}{
        		    $l \gets k$
        		}
                \lIf{$m' > m$}{
        		    $r \gets k-1$
        		}
                \lIf{$m' = m$}{
        		    $r \gets k$
        		}
        }
        $m \gets m'$\;
    }
    \Return{$dual(X,e_k)$}

	\caption{$binsApprox(X,m)$}   
	\label{alg:naive}
\end{algorithm}

\begin{proposition}\label{the:correctnessBinsearch}
	    $binsApprox(X,m) \in \underset{|X'| \leq m }{\operatorname{arg\,min}}\, d_K(X,X')$
\end{proposition}
\begin{proof}
Let $E=\{ Pr(x_1 {<} X {\leq} x_2 ) \colon x_1,x_2 \in \support(X) \cup \{-\infty, \infty\}, x_1 < x_2  \}$
and let
$E'= \{ d_K(X,X') \colon  |X'|\leq m, F_X(x)
\leq F_{X'}(x) \text{ for all } x \}.$
Since lines 10-18 are describing a binary search to find the smallest $e\in E$ such that $|dual(X,e)|\leq m$, we only need to prove that $\min E' \in E$.


Let $\varepsilon = \min E'$ and assume by contradiction that exists $\varepsilon \notin E$. 
Let $X' = dual(X,\varepsilon)$. By definition, there is $x \in \support(X)$ in which $F_{X'}(x) - F_{X}(x) = \varepsilon$. Let $x_- = \max \{ x' \colon x' \leq x, x'\in \support(X')\}$. 
If $x \in \support(X')$ we have that $x_-=x$. Otherwise we have that $Pr(x_- < X' \leq x)=0$. In both cases we get that $F_{X'}(x_{-}) = F_{X'}(x)$. Plugging this in the previous inequality gives us that  
$F_{X'}(x_{-}) - F_{X}(x) =\varepsilon$. Since $X'$ is constructed by the $dual$ algorithm, $x_{-} \in \support(X)$ and $\exists t \in \support(X) \colon F_{X'}(x_-)=F_{X}(t)$. Therefore, $F_{X'}(x_{-}) - F_{X}(x) \in E$ in contradiction to  $\varepsilon \notin E$.

\end{proof}

\begin{proposition}\label{the:complexityBinsearch}
	The $binsApprox(X,m)$ algorithm runs in time $O(n^2\log(n))$, using $O(n^2)$ memory where $n=|X|$.
\end{proposition}
\begin{proof}
	In the first part of the algorithm, lines 2-8, we construct the set $E$ which takes $O(n^2)$ run-time. In the second part of the algorithm, line 9, we sort the set $E$ which takes $O(n^2\log(n^2)) = O(n^2\log(n))$  run-time. The third part of the algorithm, lines 10-19, describes a binary sort over the set $E$ where in each step of the sorting we run the $dual$ algorithm, which takes $O(n \log(n^2)) = O(n \log(n))$ run-time. We run this part twice then we get $O(2n \log(n))$. All together, the run-time complexity is $O(n^2+n^2\log(n)+2n \log(n)) = O(n^2\log(n))$ and the memory complexity is for storing the set $E$ which is $O(n^2)$.
\end{proof}

Towards an improved algorithm let us introduce the matrix $E=(e_{i,j})_{i,j=1}^\infty$ defined by:
	$$e_{i,j} = \begin{cases}
	c_i & \text{if } i \leq n \wedge  j = n + 1; \\
	c_i - c_{n+1-j}  & \text{if } i < n \wedge j \leq n \wedge i{+}j\geq n; \\
	0  & \text{if }  i{+}j < n; \\
	1  & \text{otherwise}.
	\end{cases}$$
Is is sorted??? I don't think so..

Towards an improved algorithm let us introduce the matrix $E=(e_{i,j})_{i,j=1}^\infty$ defined by:
	$$e_{i,j} = \begin{cases}
	1-c_{n+1-j}   & \text{if } j \leq n \wedge  i = n ; \\
	c_i & \text{if } i \leq n \wedge  j = n + 1; \\
	( c_i - c_{n+1-j})/2  & \text{if } i < n \wedge j \leq n \wedge i{+}j\geq n; \\
	0  & \text{if }  i{+}j < n; \\
	1  & \text{otherwise}.
	\end{cases}$$
Where $c_1,\dots,c_n$ are as in the first line of Algorithm~\ref{alg:naive}. It is easy to see that the set of values in this matrix are the elements of the set $E$ in Algorithm~\ref{alg:naive}. An additional useful fact is that $E$ is a sorted matrix:
\begin{lemma}\label{the:sortedMatrix}
	If $i\leq i'$ and $j\leq j'$ then $e_{i,j}\leq e_{i',j'}$.
\end{lemma}
\begin{proof}
Since the $c_i$s are monotonically increasing, $\forall i,j$, $ 0\leq c_i - c_{n+1-j}\leq c_i - c_{n+1-(j+1)} $ and $ \forall i,j$, $0 \leq c_i - c_{n+1-j}\leq c_{i+1} - c_{n+1-j} $, moreover, 0 is the minimal value of $E$. The elements in the last row and the last column in $E$ are keeping the terms of sorted matrix. 
It is suffice to compare column $n$ with only column $n+1$ and show that $c_i\geq ( c_{i} - c_{n+1-n})$. After some manipulation we get that $c_{1}\geq 0$ which is true.
\end{proof}

The fact that this matrix is sorted allows us to use the saddleback search algorithm listed as Algorithm~\ref{alg:saddleback}. The algorithm starts at the top right entry of the matrix $(e_{i,j})_{i=1..n,j=1..(n+1)}$ and traverses it as follows. If it hits an entry $e$ such that $|dual(X,e)| \leq m$ it goes left, otherwise it goes down. This assures that the minimal $e$ such that $|dual(X,e)| \leq m$ is visited after at most $n+1$ steps. The optimal random variable is found by brute-force search over the visited entries.

\begin{algorithm}
	\DontPrintSemicolon
	Let $\{(x_i, c_i)\}_{i=1}^n$ be such that $c_i=Pr(X \leq x_i)$ and $\supp(X)=\{x_1 < \cdots < x_n\}$.\;
    $i \gets 1$, $j \gets n$+1, $S \gets \emptyset$\;

	\While{$i < n \wedge j \geq 1$}
	{	
		$m' \gets |dual(X,e_{i,j})|$\;
	
		\lIf {$m'  \leq m$}
		{
			$j \gets j - 1$
		}
		\lIf {$m' > m$} {
			$i \gets i + 1$
		}
		\lIf {$m' \geq m$} {
			$S \gets S \cup (m',e_{i,j})$			
		}

	}
	$e \gets \min\{e\colon (m',e) \in \underset{(m',e)\in S}{\operatorname{arg\,max}}\, m'\}$\;
	\Return{$dual(X,e)$}
	
	\caption{$sdlbkApprox(X,m)$}   
	\label{alg:saddleback}
\end{algorithm}

\begin{proposition}\label{the:correctnessSaddleback}
	   $sdlbkApprox(X,m) {\in} \underset{|X'| \leq m }{\operatorname{arg\,min}}\, d_K(X,X')$
\end{proposition}
\begin{proof}
The algorithm traverses all the frontier between those $e$s that satisfy $|dual(X,e)| \leq m$ and those that do not satisfy it. Since all the entries that satisfy the condition are recorded in $S$ and considered in the brute-force phase in line 7, the minimal satisfying $e$ is found.
\end{proof}

\begin{proposition}\label{the:complexitySaddleback}
	The $sdlbkApprox(X,m)$ algorithm runs in time $O(n^2)$, using $O(n)$ memory where $n=|X|$.
\end{proposition}
\begin{proof}
	At each step of the loop in lines 3-6, either $j$ decreases or $i$ increases, thus the loop can be executed at most $2n$ times. Since we execute $dual$ once in a loop round, the total time complexity is $O(n^2)$. Storing visited states in $S$ on line 6 requires  $O(n)$ memory.
\end{proof}

The problem with the saddleback algorithm is that it needs to run $dual$ at every step so it has a quadratic time complexity. Since we cannot find the required entry of the matrix in less than $n$ steps, we can only reduce the complexity by proposing an algorithm that does not execute $dual$ in all the steps, only in $\log(n)$ of them. Such an algorithm, based on Section~2.1 of~\cite{Fournier2011}, is listed as Algorithm~\ref{alg:linear}. The algorithm maintains a set $S$ of sub-matrices of $(e_{i,j})_{i=1..2^{\lceil \log_2(n) \rceil},j=1..2^{\lceil \log_2(n+1) \rceil}}$. At each round of its execution, each sub-matrix is split to four and then about three quarters of the matrix are discarded. At the end, at most four scalar matrices remain containing the index of the entry we seek. Note that this algorithm runs on a matrix that, in the worst case, can be almost four times bigger than the matrix traversed by the saddleback algorithm. This, of course, does not affect the asymptotic  complexity, but it may matter when dealing with relatively small random variables.

\begin{algorithm}
		\DontPrintSemicolon
		Let $\{(x_i, c_i)\}_{i=1}^n$ be such that $c_i=Pr(X \leq x_i)$ and $\supp(X)=\{x_1 < \cdots < x_n\}$.\;

		$S \gets  \{ ((1, 1),
		( 2^{\lceil \log_2(n) \rceil},2^{\lceil \log_2(n+1) \rceil})) \}$, $S' \gets \emptyset$\;

		\While{$S \neq S'$}
		{
			$S' \gets S$,
			$S \gets  \bigcup_{s \in S} split(s)$\;
			
			$e^- {\gets} median(\{ e_{i_1,j_1} \colon ((i_1,j_1),(i_2,j_2)) \in S \}) $\;		
			$e^+ {\gets} median(\{ e_{i_2,j_2} \colon ((i_1,j_1),(i_2,j_2)) \in S \}) $\;			

			\For{$ e \in \{ e^-,e^+ \}$} {
				$m' \gets |dual(X,e)|$\;
	
				\If {$m' \leq m$} {
					$S \gets$ $S \setminus \{ ((i_1,j_1),(i_2,j_2)) \in S \colon e_{i_1,j_1} > e \}$ \;
				}
			
				\Else 
				{
					$S \gets$ $S \setminus \{ ((i_1,j_1),(i_2,j_2)) \in S \colon e_{i_2,j_2} \leq e \}$ \;
				}
			}
		}
	
		$S'' \gets \{(|dual(X,e_{i,j})|,e_{i,j}) \colon ((i,j),(i,j)) \in S\}$\;	
	
		$e \gets \min\{e\colon (m',e) \in \underset{(m',e)\in S''}{\operatorname{arg\,max}}\, m'\}$\;
		\Return{$dual(X,e)$}

		\caption{$linApprox(X,m)$}   
		\label{alg:linear}
	\end{algorithm}
{
	\begin{function}
		\DontPrintSemicolon
		
		$j^- \gets \lfloor (j_1+j_2)/2 \rfloor$,
		$j^+ \gets \lceil (j_1+j_2)/2 \rceil$\;
		$i^- \gets \lfloor (i_1+i_2)/2 \rfloor$,
		$i^+ \gets \lceil (i_1+i_2)/2 \rceil$\;

		\Return{\hspace{4cm}$
			\{ ((i_1, j_1), (i^-, j^-)),
			((i_1, j^+), (i^-, j_2))$,
			$((i^+, j_1), (i_2, j^-)),
			((i^+, j^+), (i_2, j_2))    \}$}
		
		\caption{{$split(((i_1,j_1),(i_2,j_2)))$}()}
		\label{proc:split}
	\end{function}
}

\begin{theorem}\label{the:correctnessLinear}
	   
   $linApprox(X,m) {\in} \underset{|X'|\leq m}{\operatorname{arg\,min}}\, d_K(X,X')$
   
\end{theorem}
\begin{proof}
In line~10 we only discard sub-matrices whose minimal entry (at the top left) is larger than an entry that we prefer. In line~12 we only discard sub-matrices whose maximal entry (at the bottom right) is smaller or equal to an entry that does not meet our condition.
\end{proof}

\begin{theorem}\label{the:complexityLinear}
	The $linApprox(X,m)$ algorithm runs in time $O(n \log(n))$, using $O(1)$ memory where $n=|X|$.
\end{theorem}
\begin{proof}
The dimension of each matrix is halved in each round, thus the loop is executed $O(\log(n))$ rounds. Since $dual$ is called one in each round and a finite number of times at the end, the time complexity is $O(n \log(n))$.
\end{proof}

\bibliographystyle{plain}
\bibliography{library,Trim_Optimum}

\end{document}